\documentclass{article}

\usepackage[preprint,nonatbib]{neurips_2026}

\usepackage[utf8]{inputenc}
\usepackage[T1]{fontenc}
\usepackage{url}
\usepackage{booktabs}
\usepackage{amsfonts}
\usepackage{amsmath,amssymb,amsthm}
\usepackage{bm}
\usepackage{graphicx}
\usepackage{textcomp}
\usepackage[table]{xcolor}
\usepackage{enumitem}
\usepackage{needspace}
\usepackage{algorithm}
\usepackage{algpseudocode}
\usepackage{nicefrac}
\usepackage{microtype}
\usepackage{cite}
\usepackage{xcolor}
\usepackage{hyperref}
\usepackage{afterpage}
\hypersetup{
    colorlinks=true,
    linkcolor=black!30!blue,  
    citecolor=black!30!blue,
    urlcolor=black!30!blue
}

\graphicspath{{./}{./figures/}{./FiguresEN/}}

\newtheorem{definition}{Definition}
\newtheorem{assumption}{Assumption}
\newtheorem{lemma}{Lemma}
\newtheorem{theorem}{Theorem}
\newtheorem{corollary}{Corollary}

\definecolor{metricpos}{HTML}{38a169}
\definecolor{metricneg}{HTML}{e53e3e}
\definecolor{hlrow}{HTML}{f7fafc}
\definecolor{accent}{HTML}{2b6cb0}

\begin{document}

\title{An Interpretable Memory Decision Controller for LLM Agents Based on Three-Signal Complementarity: Decoupling Confidence and Consistency}

\author{%
  \textbf{Yiming Zhang}, \enspace \textbf{Jinghong Zhang}, \enspace \textbf{Haoran Zhao}, \\
  \textbf{Yiren Ma}, \enspace \textbf{Chunlei Zhao}\thanks{Corresponding author.} \\
  School of Computer Science and Engineering, Tianjin University of Technology \\
  Tianjin, China \\
  \texttt{yimingzhangre@outlook.com}, \enspace \texttt{zcltjut@126.com}
}

\maketitle

\begin{abstract}
Memory systems for large language models have focused predominantly on
efficient retrieval, whereas the decision of whether retrieved memories
should be trusted has received comparatively little attention. When the
memory store contains conflicting positions, standard retrieval-augmented
generation (RAG) blindly injects memories and amplifies hallucinations: in models susceptible to memory injection, the
RAG hallucination rate under conflicting memories is markedly higher than
that of a memory-free baseline. Inspired by memory signaling mechanisms in
the prefrontal cortex, we propose the Memory Decision Layer (MDL), a
zero-parameter memory decision controller situated between the retrieval
and generation stages. Its core is a three-signal complementary encoder
that fuses relevance, reliability, and task risk through QR-based
orthogonal subspace projection and a meta-working-memory signal into an
interpretable decision representation that quantifies the trustworthiness
of retrieved memories. Building on this encoder, MDL explicitly decouples
confidence from consistency and introduces risk inversion and explicit
abstention. Evaluations on mainstream large language models and multiple
open-source datasets show that MDL reduces the hallucination rate under
conflicting memories by about 56.04\% in general scenarios and approaches
zero hallucination in high-risk scenarios. The controller is fully
white-box: it relies purely on geometric operations, requires no trained
parameters, and adds only about 0.14~ms per decision---roughly
$50\times$ faster than the embedding-retrieval step that precedes it and
four to five orders of magnitude faster than an LLM self-evaluation call.
\end{abstract}

\begin{figure}[!t]
\centering
\includegraphics[width=\textwidth]{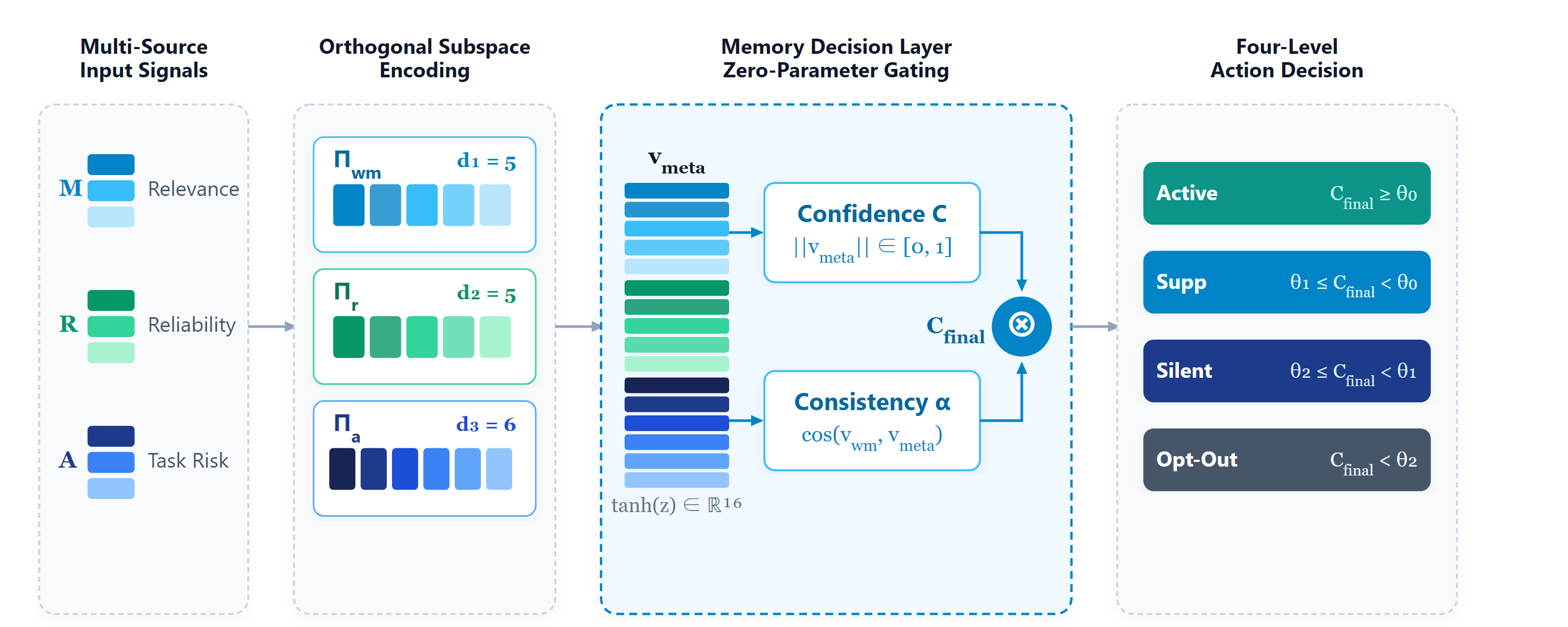}
\caption{\textbf{Complete four-stage architecture of the Memory Decision Layer
(MDL):} $\mathbf{1}$ multi-source input signals, $\mathbf{2}$ orthogonal
subspace encoding, $\mathbf{3}$ the memory decision layer, and
$\mathbf{4}$ four-level action decision.\label{fig:arch}}
\end{figure}

\section{Introduction}
\label{sec:introduction}

Large language model (LLM) agents rely heavily on external
memory systems in multi-turn interaction and complex task solving.
Retrieval-augmented generation (RAG) has become the dominant memory
paradigm: knowledge is stored in an external corpus, semantically similar
memories are retrieved, and they are injected into the generation context.
Deciding correctly whether the retrieved memories can be trusted is a key
prerequisite for the reliability of agent outputs.

Retrieving a highly relevant memory, however, does not imply that the
memory is trustworthy. An outdated pharmacopoeia is adopted because of its
high relevance, yet its advice is obsolete and dangerous; deprecated API
documentation is emitted because it appears so frequently; New York case
law is mistakenly cited in a California legal consultation. These cases
share a common signature: high retrieval confidence but low substantive
consistency between the memory and the task direction---the retriever
only answers \emph{which} memory is most relevant, but never whether that
memory is trustworthy.

Worse still, when the memory store contains conflicting positions,
standard RAG amplifies rather than mitigates hallucinations, a failure
mode systematically documented in the knowledge-conflict literature
\cite{conflictbank2024,kcsurvey2024}. In end-to-end experiments on the
TruthfulQA dataset \cite{lin2022truthfulqa}, when the memory store
contains both correct answers and common misconceptions, the
hallucination rate of RAG reaches 53.0\%, significantly higher than the
23.0\% of the memory-free baseline ($p = 0.007$). The root cause is
\emph{retrieval as adoption}: the system unconditionally injects every
retrieved memory without any trust assessment, and the model is misled by
erroneous memories in roughly half of the cases. This quantitative gap
indicates that the post-retrieval memory decision stage is an
indispensable component of any trustworthy memory system.

Most existing studies focus on retrieving more efficiently or on
repairing outputs after generation, leaving the trust-decision stage
between retrieval and generation largely vacant. GraphRAG
\cite{edge2024graphrag}, MemGPT \cite{packer2024memgpt}, and Mem0
\cite{chhikara2025mem0} optimize storage and retrieval; HippoRAG
\cite{hipporag2024} organizes human-like long-term memory in a graph
structure. Self-assessment approaches such as Self-RAG
\cite{asai2024selfrag} and Reflexion \cite{shinn2023reflexion} defer
trust verification until after generation, by which time erroneous
memories have already contaminated the context and can only be remedied,
not prevented; Huang \emph{et al.} \cite{selfcorrupt2024} further show
that LLMs can hardly self-correct reasoning errors without external
feedback. Post-generation verification alone therefore cannot
fundamentally stop erroneous memories from being injected.

Filling this gap requires judging whether to adopt or abstain---before
memory injection and using only retrieval-stage signals. This requirement
maps naturally onto a recent neuroscience finding: Ning \emph{et al.}
\cite{ning2025metaWM} observed in the macaque prefrontal cortex (PFC)
that metacognitive memory judgments are not driven by a single
memory-strength signal; instead, neural populations encode complementary
components---working-memory strength, trial history, and arousal
level---in nearly independent functional subspaces and fuse them into a
meta-working-memory signal that guides behavior and abstention decisions.
This computational motif---assessing trustworthiness before consuming
information---provides the legitimacy for inserting a memory decision
layer between retrieval and generation.

Inspired by this finding, we propose the Memory Decision Layer (MDL)
(architecture shown in Fig.~\ref{fig:arch}): it sits between the
retrieval and generation layers and answers, with zero trained
parameters, whether retrieved memories are trustworthy and whether they
should be adopted. The core of MDL is a \textbf{three-signal
complementary encoder}: it combines three complementary signal
channels---relevance ($M$), reliability ($R$), and task risk
($A$)---projects them onto orthogonal subspaces, and fuses them with a
meta-working-memory signal into a unified interpretable decision
representation $\mathbf{v}_{\mathrm{meta}}$. The risk signal $A$ is a
task-level engineering extension that captures the high-stakes
consequences of domains such as healthcare, law, and finance. On top of
this encoder, MDL further exploits the explicit decoupling of confidence
$C$ (norm) and consistency $\alpha$ (cosine), value encoding, and risk
inversion ($s_A^{\mathrm{inv}} = 1 - A$) to strengthen auditability and
safety redundancy---among these properties, the $C$--$\alpha$ decoupling
is the most prominent. MDL proceeds in four progressive stages:
(1)~signal extraction and risk inversion; (2)~orthogonal subspace
projection; (3)~meta-working-memory fusion with $C$--$\alpha$
decoupling; and (4)~four-level action mapping. Purely geometric
operations guarantee ultra-low latency, and the explicit scalars $C$ and
$\alpha$ provide an interpretable attribution-and-audit interface.

The main contributions of this article are as follows:
\begin{enumerate}[label=\textbf{(\arabic*)}, leftmargin=2.5em]
\item \textbf{Three-signal complementary encoder}: We propose a
zero-parameter memory decision encoder that integrates relevance ($M$),
reliability ($R$), and task risk ($A$), placed between the retrieval and
generation layers; orthogonal subspace coding and meta-working-memory
fusion jointly characterize the trustworthiness of retrieved memories. An
information-theoretic decomposition shows that fusing the three signals
yields an observable synergy gain, and that the $M$ signal provides
irreplaceable safety-abstention value in the low-relevance regime.

\item {\bfseries $C$--$\alpha$ decoupling mechanism}: On top of the
encoder, we design orthogonal subspace projection and a gating mechanism
that expose confidence $C$ (norm) and directional consistency $\alpha$
(cosine) as two independently auditable scalars. Combined with
risk-inversion encoding ($s_A^{\mathrm{inv}} = 1 - A$), high-risk or
conflicting memories spontaneously trigger abstention.

\item \textbf{Zero-training, ultra-low-latency deployment}: The
controller relies entirely on geometric operations (QR decomposition,
vector norms, and cosine angles), requires no trained parameters,
achieves a single-decision latency of 40.1~$\mu$s (about 0.14~ms
including lexical signal aggregation), and outputs the
$C$ and $\alpha$ scalars as an explicit attribution-audit interface.

\item \textbf{Cross-model end-to-end hallucination suppression}: We
conduct extensive validation on TruthfulQA and HaluEval. MDL reduces the
hallucination rate under conflicting memories from 53.0\% (RAG) to
23.3\% ($p = 0.014$) in general scenarios and achieves near-zero
high-risk hallucination rates across mainstream LLMs including
gemma-4-E4B-it, deepseek-v4-flash, and gemini-3-flash-preview.
\end{enumerate}

\begin{figure}[!t]
\centering
\includegraphics[width=\textwidth,height=0.2\textheight,keepaspectratio]{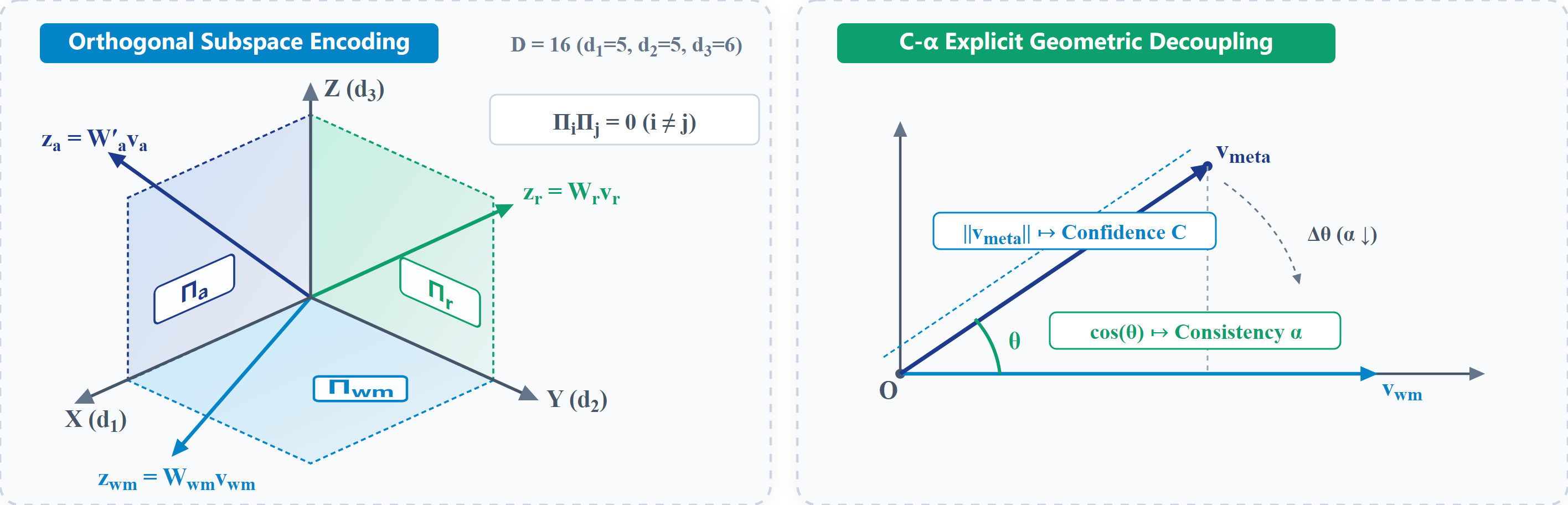}
\caption{\textbf{Illustration of the four-stage gated decision and
action-mapping mechanism.}\label{fig:figure2}}
\end{figure}

\section{Related Work}
\label{sec:related}

This section classifies related work along a single axis---the stage at
which trust decisions about retrieved memories are made---into four
categories, and closes each category with a precise statement of the
difference between it and our work. Table~\ref{tab:innovation-matrix}
contrasts the proposed MDL with representative methods along capability
dimensions.

\subsection{Memory Organization and Retrieval Optimization Before Generation}
These systems optimize memory storage and retrieval before generation,
but make no decision on the trustworthiness of already-retrieved
memories. Standard RAG drives generation with top-$k$ semantic recall;
GraphRAG \cite{edge2024graphrag} augments retrieval with a graph
structure; HippoRAG \cite{hipporag2024} and its extension
\cite{hipporag2_2025} organize knowledge graphs with personalized
PageRank into human-like long-term memory; MemGPT
\cite{packer2024memgpt} manages the context through paged main/external
memory; MemoryBank \cite{memorybank}, Generative Agents
\cite{genagents}, Larimar \cite{larimar2024}, ExpeL \cite{expel2024},
Mem0 \cite{chhikara2025mem0}, A-MEM \cite{amem2025}, and LongMem
\cite{wang2023longmem} extend memory systems from the perspectives of
forgetting curves, memory streams, knowledge editing, experience
distillation, agentic memory organization, and long-term capacity,
respectively; Adaptive-RAG \cite{adaptiverag2024} adapts the retrieval
strategy to question complexity. The shared limitation of these systems
is that the decision relies entirely on a single retrieval score: when
the retrieval score is high but directionally wrong, the memory is still
adopted unconditionally.

\subsection{In-Generation and Post-Generation Self-Assessment}
These methods self-assess outputs during or after generation, but by
then erroneous memories have already been injected into the context;
they can only remedy, not prevent. Self-RAG \cite{asai2024selfrag}
self-evaluates retrieved content and generation quality with reflection
tokens; Reflexion \cite{shinn2023reflexion} refines strategies with
verbal feedback after failures; CRITIC \cite{critic2024} invokes
external tools to verify and revise outputs; RA-ISF \cite{raisf2024}
iteratively filters retrieval through self-feedback. More fundamentally,
Huang \emph{et al.} \cite{selfcorrupt2024} show that LLMs can hardly
self-correct reasoning errors without external feedback, and Zhang
\emph{et al.} \cite{darkside2025} further reveal the underlying failure
mechanism. Such methods therefore cannot eliminate the contamination
already caused at injection time.

\subsection{Post-Hoc Hallucination Detection}
These methods detect unfaithful content after generation completes;
they are designed for detection rather than prevention. RAGTruth
\cite{ragtruth2024} constructs a large-scale, fine-grained,
human-annotated hallucination corpus for RAG; semantic entropy
\cite{semanticentropy2024} clusters semantically equivalent answers
sampled multiple times and computes the entropy, demonstrating in
\emph{Nature} that hallucinations can be detected without training data;
Pacchiardi \emph{et al.} \cite{liar2024} identify unfaithful generations
in black-box models by asking unrelated questions. These approaches
likewise operate after contamination, and schemes such as semantic
entropy require multiple samples at a high per-query cost, in contrast
to the zero-parameter, single-shot, microsecond-level decision pursued
in this article.

\subsection{Trust Decision After Retrieval and Before Generation}
The methods above handle memories either before generation or after it;
what is missing is precisely the decision window between retrieval and
generation. ReAct \cite{react} and Toolformer \cite{toolformer} address
reasoning--tool coordination without evaluating memory trustworthiness;
the memory-routing scheme DAM \cite{sun2025dam} lacks risk awareness;
and recent efforts such as MemoryScope \cite{memoryscope}, MemoryLLM
\cite{memoryllm}, and MEMO \cite{memo} still focus on memory
organization, encoding, and evaluation, without explicitly decoupling
confidence from consistency. This article treats trust in retrieved
memories as an independent decision stage: it explicitly decouples
confidence $C$ and consistency $\alpha$ through zero-training geometric
operations and combines risk inversion with explicit abstention, filling
this gap. Studies on knowledge conflicts
\cite{conflictbank2024,kcsurvey2024} further confirm that conflicting
memories are a systematic source of failure, underscoring the necessity
of this decision stage.

\begin{table}[!t]
\centering
\caption{Innovation matrix: existing methods vs.\ the proposed MDL.}
\label{tab:innovation-matrix}
\footnotesize
\setlength{\tabcolsep}{3pt}
\begin{tabular}{lcccc}
\toprule
Capability & RAG & MemGPT \cite{packer2024memgpt} & Reflexion \cite{shinn2023reflexion} & Ours (MDL) \\
\midrule
$C$--$\alpha$ decoupling & Implicit & $\times$ & Implicit & \textcolor{metricpos}{\textbf{Explicit}} \\
Risk opt-out & $\times$ & Partial & $\times$ & \textcolor{metricpos}{\textbf{Explicit}} \\
Subspace attribution & $\times$ & $\times$ & $\times$ & \textcolor{metricpos}{\textbf{Explicit}} \\
Direction auditing & $\times$ & $\times$ & $\times$ & \textcolor{metricpos}{\textbf{Explicit}} \\
\bottomrule
\end{tabular}
\end{table}

\section{Method}
\label{sec:method}

\subsection{Overview and Core Design Principles}

The MDL decision controller sits between the retrieval and generation
layers and is formalized as $\pi : (q, m, R, A) \to a$. As shown in
Fig.~\ref{fig:arch}, the controller proceeds in four progressive stages:
(1)~signal extraction and risk inversion; (2)~orthogonal subspace
projection; (3)~meta-working-memory fusion with $C$--$\alpha$
decoupling; and (4)~four-level action mapping. The controller follows
five design principles: interpretability first, geometric orthogonal
decoupling, bounded state activation, explicit risk awareness, and
zero-parameter ultra-low latency.

\subsection{Signal Encoding and Orthogonal Subspace Fusion}

The controller first extracts the relevance
$M = \max_j \cos(\mathbf{e}_q, \mathbf{e}_{m_j})$, i.e., the maximum
cosine similarity between the sentence embedding of the query and those
of the candidate memories (sentence-transformer
\emph{all-MiniLM-L6-v2}); the reliability
\begin{equation*}
R = \mathrm{clip}\!\bigl(\bar{s}_{\mathrm{rel}} \cdot (1 - \phi)^2,\, 0,\, 1\bigr),
\end{equation*}
i.e., the product of the mean pairwise embedding similarity
$\bar{s}_{\mathrm{rel}}$ among the relevant memories
($\cos(\mathbf{e}_q, \mathbf{e}_{m_j}) > 0.3$; $R = 0.5$ if fewer than
two memories are relevant) and the squared complement of the
stance-conflict rate $\phi$---the fraction of relevant memory pairs
flagged as positionally conflicting by a lexical conflict detector
covering negation, yes/no polarity, and antonym pairs; and the
risk-inversion encoding $s_A^{\mathrm{inv}} = 1 - A$ (which
converts high risk into low activation). The task-risk coefficient $A$
is assigned by a graded category-to-risk mapping (TruthfulQA) or by
health/legal/finance and science keyword matching (HaluEval, with values
in $\{0.20, 0.50, 0.85\}$); queries in the top risk tier (e.g., health,
law, finance, and conspiracy categories) receive $A \ge 0.70$. Each scalar $s \in [0, 1]$ is
mapped to a $D=16$-dimensional value vector
\begin{equation*}
\mathbf{v}(s) = s \mathbf{1}_D + \sum_{i=1}^{D} \exp\!\bigl(-(s - c_i)^2 / 2\sigma^2\bigr) \mathbf{e}_i
\quad (\sigma = 0.20),
\end{equation*}
yielding the channel vectors $\mathbf{v}_{\mathrm{wm}}$, $\mathbf{v}_r$,
and $\mathbf{v}_a$.

Based on the complete orthogonal projectors
$\Pi_k = P_k P_k^{\top}$ generated by QR decomposition (dimension
allocation $d_1=5$, $d_2=5$, $d_3=6$), the subspace weight matrices are
constructed as
\begin{equation}
\label{eq:Wk}
W_k = s_k \Pi_k + \epsilon_c \textstyle\sum_{j \neq k} \Pi_j,
\end{equation}
and the working-memory (relevance) pathway carries a second-order risk
modulation: the encoded risk vector $\mathbf{v}_a$ yields a bounded
scalar gain
\begin{equation*}
g_A = 0.5 + 0.5\,\langle \mathbf{v}_a \rangle,
\end{equation*}
where $\langle \cdot \rangle$ denotes the dimension-wise mean. Because
the risk-inversion encoding $s_A^{\mathrm{inv}} = 1 - A$ lowers the
magnitude of $\mathbf{v}_a$ as the risk level rises, $g_A$ shrinks
accordingly, dampening the relevance-driven confidence precisely when
risk is high. Fusion yields
the bounded intermediate decision representation
$\mathbf{v}_{\mathrm{meta}} = \tanh(g_A W_{\mathrm{wm}} \mathbf{v}_{\mathrm{wm}} + W_r \mathbf{v}_r + W_a \mathbf{v}_a + \mathbf{b}) \in \mathbb{R}^{16}$.

\subsection{Confidence--Consistency Decoupling, Gating, and Action Decision}

From $\mathbf{v}_{\mathrm{meta}}$, the confidence norm $C$ and the
directional-consistency cosine $\alpha$ are derived in parallel:
\begin{equation}
\label{eq:Calpha}
\begin{split}
C &= \mathrm{clip}\!\left(\frac{\|\mathbf{v}_{\mathrm{meta}}\| - n_{\min}}{n_{\max} - n_{\min}},\, 0,\, 1\right), \\
\alpha &= \frac{\mathbf{v}_{\mathrm{wm}} \cdot \mathbf{v}_{\mathrm{meta}}}{\|\mathbf{v}_{\mathrm{wm}}\| \cdot \|\mathbf{v}_{\mathrm{meta}}\|}.
\end{split}
\end{equation}
Gating modulation with $\alpha$ gives
$C_{\mathrm{final}} = C \cdot (0.3 + 0.7 \alpha)$ (mechanism in
Fig.~\ref{fig:figure2}). According to the dynamic threshold vector
$\boldsymbol{\theta} = (\theta_0, \theta_1, \theta_2)$, the gated confidence $C_{\mathrm{final}}$
is mapped to the four-level discrete action
$a \in \{\text{Active}, \text{Supp}, \text{Silent}, \text{Opt-Out}\}$.

\subsection{Algorithm and Parameter Settings}

The complete configuration is summarized in Table~\ref{tab:hyper}, and
the inference procedure is given in Algorithm~\ref{alg:1}. The
thresholds $\boldsymbol{\theta}=(\theta_0,\theta_1,\theta_2)$ are
calibrated on a held-out calibration split of 80 questions per dataset,
disjoint from the evaluated questions: $C_{\mathrm{final}}$ is computed
for every calibration question, and the primary threshold $\theta_0$ is
selected by grid search to maximize agreement with the oracle actions on
this split, with the two lower thresholds $\theta_1$ and $\theta_2$
derived from $\theta_0$ with a small offset. The normalization
parameters $n_{\min}$ and $n_{\max}$ are likewise taken from the
calibration split: they are the minimum and maximum of the raw norm
$\|\mathbf{v}_{\mathrm{meta}}\|$ over the calibration questions.

\begin{table}[!htbp]
\centering
\caption{Complete hyperparameter configuration of the MDL controller.}
\label{tab:hyper}
\footnotesize
\setlength{\tabcolsep}{4pt}
\begin{tabular}{llp{0.55\textwidth}}
\toprule
Hyperparameter & Value & Selection rationale \\
\midrule
$D$              & 16              & Accommodates three subspaces (5+5+6) with expressive margin; too small saturates the geometry, too large raises latency \\
$(d_{\mathrm{wm}}, d_r, d_a)$ & $(5, 5, 6)$ & One extra dimension for the risk subspace to carry second-order risk modulation; dimensions matched to signal complexity (validated by ablation) \\
$\sigma$         & 0.20            & Sharpness of value encoding; too small overfits single points, too large loses directional expression; 0.20 lies on the directional-separability plateau \\
$\epsilon_c$     & 0.10            & Ensures $\epsilon_c<\min_k s_k$, so principal gains dominate and monotonicity holds (Theorem~\ref{thm:monotonicity}) \\
$s_{\mathrm{wm}} / s_r / s_a$ & $1.0 / 0.7 / 0.5$ & Heuristic allocation with relevance as the primary direction, reliability as secondary, and risk as tertiary \\
$g_A$ coefficients & $(0.5, 0.5)$ & Bounded second-order risk modulation of the relevance pathway ($g_A = 0.5 + 0.5\langle\mathbf{v}_a\rangle$); removing it costs 1.7~pp (Table~\ref{tab:ablation-baselines}) \\
$\alpha$ gating coefficients & $(0.3, 0.7)$ & Grid search (Fig.~\ref{fig:matrix_ablation}(b)) shows a broad high-accuracy ridge at $g_a{+}g_b \approx 0.9$ (up to 64.4\%); the default $(0.3,0.7)$ lies in the adjacent stable region \\
Random seeds & 2--3 per backend & End-to-end: seed$\in\{0,1,2\}$ (deepseek-v4-flash) and $\{0,1\}$ (gemma-4-E4B-it); sensitivity analyses: 10 bootstrap resamples; mean$\pm$std reported \\
\bottomrule
\end{tabular}
\end{table}

\begin{algorithm}[t!]
\caption{Inference and Attribution Algorithm of the MDL Agent Memory Decision Controller}
\label{alg:1}
\begin{algorithmic}[1]
\Require Query $q$, candidate memory $m$, risk coefficient $A \in [0, 1]$
\Require Subspace dimensions $(d_{\mathrm{wm}}, d_r, d_a)$, hyperparameters $(\sigma, \epsilon_c, \mathbf{s})$
\Require Calibration parameters $n_{\min}, n_{\max}$, thresholds $\boldsymbol{\theta}$
\Ensure Action $a$, attribution scalars $(C, \alpha, C_{\mathrm{final}})$
\State $Q \gets \mathrm{QR}(R_{\mathrm{rand}})$; construct $\Pi_k \gets P_k P_k^{\top}$
\State Extract $M, R$; $s_A^{\mathrm{inv}} \gets 1 - A$; encode $\mathbf{v}_{\mathrm{wm}}, \mathbf{v}_r, \mathbf{v}_a$
\State $g_A \gets 0.5 + 0.5\,\langle\mathbf{v}_a\rangle$; compute $\mathbf{v}_{\mathrm{meta}} \gets \tanh(g_A W_{\mathrm{wm}}\mathbf{v}_{\mathrm{wm}} + W_r\mathbf{v}_r + W_a\mathbf{v}_a + \mathbf{b})$
\State Derive $C \gets \mathrm{clip}\!\left(\frac{\|\mathbf{v}_{\mathrm{meta}}\| - n_{\min}}{n_{\max} - n_{\min}},\, 0,\, 1\right)$
\State Derive $\alpha \gets \cos(\mathbf{v}_{\mathrm{wm}}, \mathbf{v}_{\mathrm{meta}})$
\State $C_{\mathrm{final}} \gets C \cdot (0.3 + 0.7\alpha)$; output $a$ according to the thresholds $\boldsymbol{\theta}$
\end{algorithmic}
\end{algorithm}

\section{Theoretical Analysis}
\label{sec:theory}

\begin{definition}[Memory Decision Problem]
Seek a function $\pi : [0, 1]^3 \to \mathcal{A}$ that, through the
intermediate representation $\mathbf{v}_{\mathrm{meta}} \in
\mathbb{R}^D$, factorizes as $\pi = \psi \circ \phi$ and maximizes the
agreement between the output and the true action.
\end{definition}

\begin{assumption}[Signal Monotonicity and Orthogonal Subspaces]
\label{ass:1}
$\|\mathrm{Enc}(s)\|$ is monotonically increasing in $s$;
$\{\Pi_k\}_{k=1}^3$ satisfy orthogonal idempotency and completeness
$\sum_k \Pi_k = I$.
\end{assumption}

\begin{lemma}[Subspace Orthogonality and Energy Preservation]
\label{lem:orth}
Under Assumption~\ref{ass:1},
$\mathrm{range}(\Pi_i) \perp \mathrm{range}(\Pi_j)$ for $i \neq j$, and
for any encoded vector $\mathbf{v}$,
$\|\mathbf{v}\|^2 = \sum_{k=1}^3 \|\Pi_k \mathbf{v}\|^2$.
\end{lemma}
\begin{proof}
Orthogonality follows directly from
$\Pi_i \Pi_j = P_i(P_i^{\top} P_j)P_j^{\top} = \mathbf{0}$, because the
subspaces are orthogonal, i.e., $P_i^{\top} P_j = \mathbf{0}$. Energy
preservation follows from the expansion over the complete orthogonal
basis and $\langle \Pi_i \mathbf{v}, \Pi_j \mathbf{v} \rangle =
\mathbf{v}^{\top} \Pi_i \Pi_j \mathbf{v} = 0$ for $i \neq j$.
\end{proof}

\begin{corollary}[Boundedness of the Cosine Consistency]
\label{cor:alpha}
By the Cauchy--Schwarz inequality, the cosine consistency satisfies
$\alpha = \cos(\mathbf{v}_{\mathrm{wm}}, \mathbf{v}_{\mathrm{meta}}) \in [-1, 1]$.
\end{corollary}

\begin{theorem}[Empirical Monotonicity of the Confidence]
\label{thm:monotonicity}
Under Assumption~\ref{ass:1} and the condition
$\epsilon_c < \min_k s_k$, $\|\mathbf{v}_{\mathrm{meta}}\|$ is
monotonically increasing in the principal gain of each signal within the
approximately linear band of $\tanh$, and hence the quantile-normalized
$C$ increases monotonically as the signals strengthen. When the signals
enter the saturation region of $\tanh$, monotonicity is no longer
guaranteed.
\end{theorem}
\begin{proof}
By the assumption, $\|\mathrm{Enc}(s)\|$ is monotonically increasing in
$s$. Within the linear band of $\tanh$,
$\|\mathbf{v}_{\mathrm{meta}}\|$ approximates
$\|\sum_k W_k \mathbf{v}_k + \mathbf{b}\|$; because
$\epsilon_c<\min_k s_k$ makes the principal projection
$s_k\Pi_k\mathbf{v}_k$ dominant while the cross-leakage terms remain
second-order small, $\|\mathbf{v}_{\mathrm{meta}}\|$ increases
monotonically with $s_k$. In the saturation region (large $\|z\|$), the
compression of $\tanh$ drives the derivative toward zero, and
monotonicity no longer holds. A complementary numerical validation is
given in Fig.~\ref{fig:matrix_ablation}(b).
\end{proof}

\noindent\textbf{Decoupling semantics and attribution value.} It should
be clarified that the $C$--$\alpha$ decoupling in this article means
\emph{independently attributable and independently auditable}, not
statistical independence: $C$ (norm) and $\alpha$ (cosine) carry two
distinct kinds of information---degree of conviction and directional
correctness---and can be read separately by downstream audit
interfaces. On the real benchmark the two are positively correlated
($r = 0.902$), yet $\alpha$ still separates correct from incorrect
decisions ($t = 3.70$, $p < 0.001$), showing that even when $C$ and
$\alpha$ are correlated, $\alpha$ carries directional information that
$C$ does not possess---a manifestation of independent auditability
rather than statistical orthogonality (correlation and auditability are
not contradictory; see Section~\ref{sec:exp-complementarity}).

\section{Experimental Evaluation}\label{sec:experiments}

\subsection{Experimental Setup and Baseline Definitions}\label{sec:baseline-def}

\paragraph*{Experimental Environment}
The signal-level analyses were conducted on an Intel Xeon Platinum 8480+
server (56 cores, 512~GB of memory, Ubuntu 22.04 LTS; the complete list
is given in Table~\ref{tab:exp-env} of
Appendix~\ref{app:env-detail}). End-to-end LLM inference used
OpenAI-compatible API endpoints for deepseek-v4-flash and
gemini-3-flash-preview, and a locally served
google/gemma-4-E4B-Instruct-128K (transformers, float16, on T4 GPUs),
all with sampling parameters temperature~$=0.7$ and top-p~$=0.9$ (at
most 150 new tokens per answer). Sample sizes are $200$ questions
$\times$ 3 seeds per dataset for deepseek-v4-flash, $150$ questions
$\times$ 2 seeds (TruthfulQA) for gemma-4-E4B-it, and $30$ questions
$\times$ 1 seed (TruthfulQA) for gemini-3-flash-preview; the HaluEval
generalization study (Table~\ref{tab:cross-dataset} and
Fig.~\ref{fig:figure3}(b)) uses the deepseek-v4-flash backend.

\paragraph*{Memory Construction and Hallucination Scoring}
For every evaluated question, a candidate memory store is constructed
from the dataset itself: the relevant memories consist of the reference
correct answer(s) together with the reference incorrect answer(s)---the
conflicting positions under study (for HaluEval, its knowledge snippet
is additionally included as a context record)---and two distractor
records drawn from other questions are added as irrelevant noise.
Free-form answers are scored as hallucinations by embedding-similarity
comparison (all-MiniLM-L6-v2) against the reference correct and
incorrect answers: an answer whose similarity to an incorrect answer
exceeds both its similarity to the correct answers and a fixed threshold
(0.5) is counted as a hallucination, whereas explicit refusal patterns
are counted as safe abstentions.

\paragraph*{Baseline Definitions}
The baselines of all experiments are uniformly defined in
Table~\ref{tab:baseline-def}:

\begin{table}[!htbp]
\centering
\caption{Precise definitions of the baselines compared in this article. All baselines use the same input signals $(M, R, A)$ as MDL (supervised baselines are trained on the same outcome-derived oracle labels).}
\label{tab:baseline-def}
\footnotesize
\setlength{\tabcolsep}{5pt}
\begin{tabular}{@{}lp{2.1in}p{2.1in}@{}}
\toprule
Baseline & {Precise definition} & {Key implementation details} \\
\midrule
B1 (no memory) & No retrieved memory is injected; the LLM generates directly & Same LLM, sampling parameters, and context \\
B2 (standard RAG) & All candidate memories of the question are injected into the context before generation & Same LLM, sampling parameters, and prompt template; memory store per question (see Memory Construction above) \\
B3 (MDL, ours) & The proposed memory decision controller placed between retrieval and generation & Same hyperparameters as Table~\ref{tab:hyper}, zero training \\
\midrule
LR & Supervised logistic regression on the three signals $(M,R,A)$ & Regularization $C=1.0$, 5-fold cross-validation \\
MLP & Multi-layer perceptron classifier with hidden layers $(64, 32)$ & ReLU activation, at most 500 training iterations, early stopping \\
XGBoost & Extreme gradient-boosted tree classifier & 100 estimators, max depth 4, 5-fold cross-validation \\
CRAG \cite{crag2024} & Retrieval evaluator grades every memory (Correct/Incorrect/Ambiguous); Correct $\to$ refine-and-generate with memories, Incorrect $\to$ corrective generation from parametric knowledge, Ambiguous $\to$ cautious combined generation & Prompt-based evaluator and corrector on the same LLM and memory stores; CRAG's web-search correction is replaced by a parametric-knowledge fallback (no external retrieval engine in the evaluation environment) \\
Self-RAG & Draft answer with memories, then reflection-token self-critique (ISREL/ISSUP/ISUSE); Irrelevant or unsupported drafts are regenerated from parametric knowledge only & Prompt-based reflection on the same LLM and memory stores; about 2.1 LLM calls per question \\
\bottomrule
\end{tabular}
\end{table}

This section aims to answer the following four core research questions
through systematic experiments:
\begin{itemize}[leftmargin=2.5em, itemsep=2pt]
\item \textbf{RQ1 (end-to-end hallucination suppression)}: Can the MDL
controller effectively suppress hallucinations caused by memory
conflicts in a real LLM generation pipeline? How does it perform in
high-risk scenarios? (Section~\ref{sec:exp-e2e})
\item \textbf{RQ2 (three-signal complementarity)}: Are the three
signals---relevance ($M$), reliability ($R$), and risk ($A$)---
complementary? Does the combined signal significantly outperform single
signals or subsets? (Section~\ref{sec:exp-complementarity})
\item \textbf{RQ3 (design and mechanism ablation)}: What does each core
mechanism---value encoding, risk inversion, and $C$--$\alpha$
decoupling---contribute? How does MDL compare with learned baselines?
(Section~\ref{sec:exp-ablation})
\item \textbf{RQ4 (computational overhead and latency)}: As middleware
between retrieval and generation, do the inference latency and
computational overhead of the MDL controller meet real-time deployment
requirements? (Section~\ref{sec:exp-latency})
\end{itemize}

\subsection{End-to-End Hallucination Suppression (RQ1)}\label{sec:exp-e2e}

\begin{figure}[!t]
    \centering
    \includegraphics[width=\textwidth]{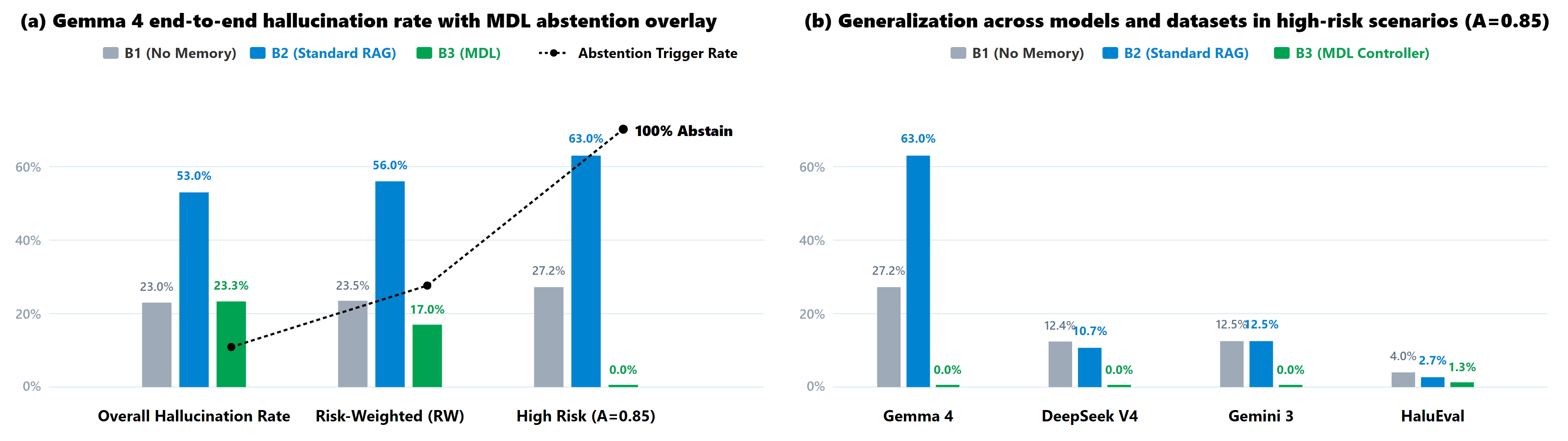}

    \caption{\textbf{Hallucination control and generalization of the MDL
    controller.} (a)~Hallucination rates (bars) of the gemma-4-E4B-it
    model under different risk levels, overlaid with the MDL abstention
    trigger rate (line) ($p = 0.014$); (b)~generalization of the MDL
    controller in high-risk scenarios ($A=0.85$) across base models
    (gemma-4-E4B-it, deepseek-v4-flash, and gemini-3-flash-preview) and
    on the HaluEval dataset.}
    \label{fig:figure3}
\end{figure}

We first evaluate the ability of the MDL controller to suppress
hallucinations in end-to-end LLM generation. As shown in
Fig.~\ref{fig:figure3}, the MDL controller exhibits excellent
hallucination control and generalization across risk levels, models, and
datasets.

\noindent\textbf{1. Hallucination amplification of standard RAG under
memory conflicts.}
As shown in Fig.~\ref{fig:figure3}(a), on the gemma-4 model, when the
memories contain conflicting positions, the hallucination rate of
standard RAG (B2) soars to 53.0\% (52.6\% under medium risk),
significantly higher than the 23.0\% of the memory-free baseline~B1
($t = 11.82$, $p = 0.007$). This confirms that blindly injecting
retrieved memories misleads the model when conflicts or noise exist and
instead aggravates hallucinations.

\noindent\textbf{2. Conflict detection and safe fallback.}
Through the $R$ signal, the MDL controller (B3) senses positional
conflicts and reduces the overall hallucination rate to 23.3\% (29.7
percentage points below RAG, $p = 0.014$); it suppresses the
hallucination rate to 27.6\% in medium-risk scenarios, halving the
RAG rate and largely eliminating the negative effect introduced by RAG.

\noindent\textbf{3. Zero hallucination in high-risk scenarios.}
In high-risk queries from healthcare, law, and finance ($A=0.85$), the
hallucination rate of standard RAG reaches 63.0\%. By contrast, MDL
triggers abstention through the risk-inversion mechanism
($s_A^{\mathrm{inv}} = 1 - A$) and achieves a 0.0\% hallucination rate
on all high-risk samples (Fig.~\ref{fig:figure3}(a)). The overall
risk-weighted (RW) hallucination rate drops to 17.0\%, the lowest among
all systems.

\begin{table}[!htbp]
\centering
\caption{Cross-dataset hallucination-rate comparison (deepseek-v4-flash, 200 questions $\times$ 3 seeds).}
\label{tab:cross-dataset}
\footnotesize
\setlength{\tabcolsep}{4pt}
\begin{tabular}{lcccc}
\toprule
& \multicolumn{2}{c}{TruthfulQA} & \multicolumn{2}{c}{HaluEval} \\
\cmidrule(lr){2-3}\cmidrule(lr){4-5}
Baseline & Overall & High-risk & Overall & High-risk \\
\midrule
B1 (no memory) & 0.160 & 0.124 & 0.083 & 0.040 \\
B2 (standard RAG) & 0.043 & 0.107 & 0.065 & 0.027 \\
\rowcolor{hlrow}
\textbf{B3 (MDL)} & \textbf{0.047} & \textbf{0.000} & \textbf{0.070} & \textbf{0.013} \\
\bottomrule
\end{tabular}
\end{table}

\noindent\textbf{4. Cross-model and cross-dataset generalization.}
As shown in Fig.~\ref{fig:figure3} and Table~\ref{tab:cross-dataset}:
(1)~\emph{Model capability and RAG resistance}: on the stronger-reasoning
deepseek-v4-flash, standard RAG shows considerable self-resistance
(overall hallucination rate down to 4.3\%). Even so, MDL still achieves
the lowest risk-weighted hallucination rate in medium- and high-risk
control (3.5\% vs.\ 5.8\% for B2 and 14.2\% for B1).
(2)~\emph{Consistent risk control across models}: on both
gemini-3-flash-preview and deepseek-v4-flash, the MDL controller again
achieves a 0.0\% hallucination rate in high-risk scenarios, showing that
its risk-awareness mechanism is independent of the reasoning capability
of the underlying LLM. (3)~\emph{Cross-dataset performance}: on the
HaluEval dataset (Table~\ref{tab:cross-dataset}), the high-risk
hallucination rate remains at an extremely low 1.3\% (vs.\ 2.7\% for B2
and 4.0\% for B1). For clarity, B1 denotes generation without any
injected memory, B2 denotes standard RAG that injects the retrieved
memories, and B3 denotes the proposed MDL controller (see the baseline
definitions in Section~\ref{sec:baseline-def}).

\noindent\textbf{5. Comparison with corrective and self-reflective
baselines.}
Finally, we compare MDL against two post-retrieval control baselines
implemented on the same LLM (deepseek-v4-flash), the same question
sample, and the same memory stores: a prompt-based \textbf{CRAG}
\cite{crag2024} variant, whose retrieval evaluator grades every memory record
(Correct/Incorrect/Ambiguous) and whose corrective step replaces
web-search correction with a parametric-knowledge fallback, and a
prompt-based \textbf{Self-RAG} variant, which drafts an answer with the
memories and then critiques it with reflection tokens (ISREL/ISSUP/
ISUSE), regenerating from parametric knowledge when the draft is judged
irrelevant or unsupported.

Table~\ref{tab:baseline-comparison} reports
the results on TruthfulQA (200 questions $\times$ 2 seeds), where the
Refusal column denotes the share of answers classified as safe
abstentions.

Three observations stand out. First, the CRAG evaluator judged
essentially every memory store \emph{Ambiguous} on this benchmark
(400/400), so its corrective path never fired and the method degenerates
to cautious RAG: its overall hallucination rate (5.0\%) matches B2, and
its high-risk rate (5.9\%) is far from the 0\% that explicit risk
awareness achieves. Second, the Self-RAG-style reflection achieves the
lowest overall rate (0.5\%), but this is bought with frequent
self-critique-driven refusals and regenerations (34/400 drafts
regenerated, 83.3\% refusal overall) and with about 2.1 LLM calls per
question---four to five orders of magnitude more latency than MDL's
single 0.14~ms geometric decision---and it yields no auditable
confidence--consistency scalars. Third, MDL matches the overall
hallucination rate of RAG while eliminating high-risk hallucinations
entirely, without any additional LLM calls.

\begin{table}[!htbp]
\centering
\caption{Baseline comparison on TruthfulQA (deepseek-v4-flash, 200 questions $\times$ 2 seeds).}
\label{tab:baseline-comparison}
\footnotesize
\setlength{\tabcolsep}{4pt}
\begin{tabular}{lccc}
\toprule
System & Overall hallucination & High-risk & Refusal \\
\midrule
B1 (no memory) & 0.155 & 0.110 & 0.380 \\
B2 (standard RAG) & 0.048 & 0.119 & 0.880 \\
CRAG (prompt-based)~\cite{crag2024} & 0.050 & 0.059 & 0.645 \\
Self-RAG (prompt-based) & \textbf{0.005} & 0.000 & 0.833 \\
\rowcolor{hlrow}
\textbf{B3 (MDL, ours)} & 0.048 & \textbf{0.000} & 0.877 \\
\bottomrule
\end{tabular}
\end{table}

\subsection{Three-Signal Complementarity and Information Analysis (RQ2)}\label{sec:exp-complementarity}

This subsection systematically evaluates the complementarity and synergy
of the three signals---relevance ($M$), reliability ($R$), and risk
($A$)---on the $3{,}600$ per-question records pooled from the end-to-end
runs of Section~\ref{sec:exp-e2e} (2 datasets $\times$ 3 seeds
$\times$ 200 questions $\times$ 3 baselines, deepseek-v4-flash backend).
Each record carries the signals $(M, R, A)$ extracted at decision time
and the realized outcome of the generation. The oracle action is derived
programmatically from the realized hallucination outcome rather than by
manual annotation: when memory was actually injected (B2, or B3 under
Active/Supp), a hallucinated answer means the memory should have been
rejected (oracle \emph{reject}) and a safe answer means it could have
been adopted (oracle \emph{adopt}); when no memory was injected (B1, or
B3 under Silent/Opt-Out), the mapping is reversed. Decision accuracy is
measured on this binary adopt/reject decision, with the intermediate
actions (Supp, Silent) mapped to \emph{reject}.
Fig.~\ref{fig:signal_interaction_heatmap} visualizes the three-signal
interactions and the confidence distributions and decision boundaries
under risk gating.

\begin{figure}[!t]
\centering
\includegraphics[width=\textwidth]{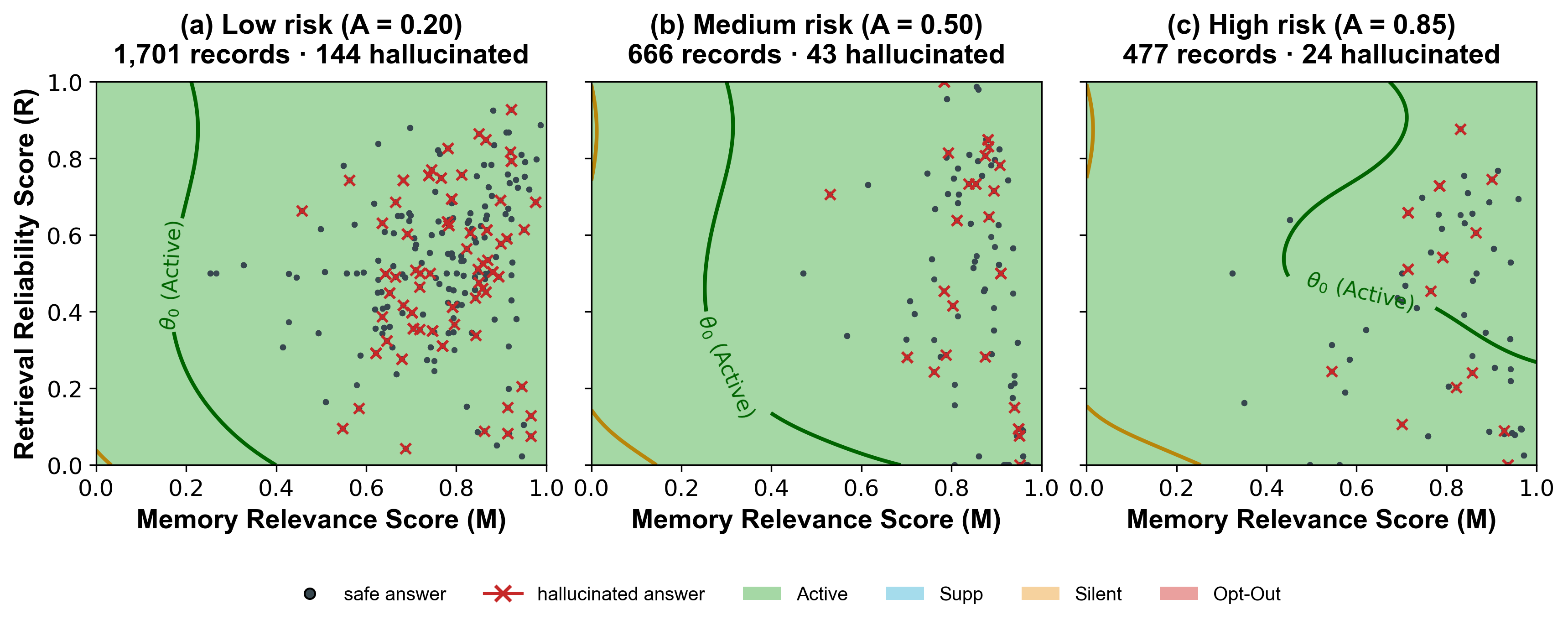}
\caption{\textbf{Decision behavior of the MDL controller over the $(M, R)$ plane
at three risk levels.} Shaded regions and contour lines show the
\emph{actual} controller's confidence $C_{\mathrm{final}}$ and
four-level action regions (computed with the deployed controller,
QR seed 42, thresholds $0.6/0.4/0.2$); dots mark the real end-to-end
records at that risk level (grey = safe answer, red cross =
hallucinated answer). As risk rises, the Active region contracts and
high-reliability memories are increasingly pushed toward
abstention.\label{fig:signal_interaction_heatmap}}
\end{figure}

\noindent\textbf{1. Accuracy gains from signal fusion}
Table~\ref{tab:main-results} reports the decision accuracy of different
signal combinations together with significance tests. The results show:
\begin{itemize}[leftmargin=2em, itemsep=2pt]
\item \textbf{Limitations of single signals}: Any single signal ($M$
only, $R$ only, or $A$ only) reaches a decision accuracy of only 0.581,
degenerating to the majority-class baseline, which indicates that a
single dimension cannot support complex memory decisions.
\item \textbf{Multi-signal fusion gains}: Decision performance improves
significantly once multiple signals are introduced. The $R{+}A$
combination reaches the best 62.7\% accuracy, and the full
$M{+}R{+}A$ combination reaches 61.2\%. Paired $t$-tests show that the
full combination significantly outperforms every single-signal
configuration ($p < 0.001$), although the paired
effect size is small (Cohen's $d_z \approx 0.10$), reflecting that most
records are decided identically across configurations.
\item \textbf{On the relation between $R{+}A$ and $M{+}R{+}A$}: the
$R{+}A$ combination is marginally but significantly \emph{higher} than
the full combination on global accuracy (62.7\% vs.\ 61.2\%,
$p = 0.001$, $d_z = 0.05$); i.e., the marginal contribution of the
$M$ signal to global accuracy is slightly negative. This does not mean
that $M$ is useless---the information-theoretic decomposition
(Table~\ref{tab:info-decomp}, Appendix~\ref{app:info}) and the
stratified safety analysis (point~3 below) show that the value of $M$
lies mainly in triggering safe abstention in the low-relevance regime,
rather than in boosting global accuracy.
\end{itemize}

\begin{table}[!htbp]
\centering
\caption{Decision accuracy of different signal combinations with paired $t$-tests ($N = 3600$ records pooled from the end-to-end runs).}
\label{tab:main-results}
\footnotesize
\setlength{\tabcolsep}{3pt}
\begin{tabular}{lcccc}
\toprule
Configuration & Accuracy & $\Delta$Acc & $d_z$ & $p$ \\
\midrule
$M$ only & 0.581 & $-$0.031 & $-$0.10 & $<$.001$^{***}$ \\
$R$ only & 0.581 & $-$0.031 & $-$0.10 & $<$.001$^{***}$ \\
$A$ only & 0.581 & $-$0.031 & $-$0.10 & $<$.001$^{***}$ \\
$M{+}R$ & 0.581 & $-$0.031 & $-$0.10 & $<$.001$^{***}$ \\
$M{+}A$ & 0.583 & $-$0.030 & $-$0.10 & $<$.001$^{***}$ \\
$R{+}A$ & \textbf{0.627} & $+$0.014 & $+$0.05 & 0.001$^{**}$ \\
\rowcolor{hlrow}
{\bfseries $M{+}R{+}A$} & \textbf{0.612} & --- & --- & --- \\
\bottomrule
\end{tabular}
\par\vspace{2pt}
{\footnotesize Paired $t$-tests over the per-record correctness indicators of each configuration versus $M{+}R{+}A$ on the same $3{,}600$ records; $d_z$ is the paired Cohen's $d$. $^{**}$~$p < 0.01$; $^{***}$~$p < 0.001$.}
\end{table}

\noindent\textbf{2. Information-theoretic synergy analysis}
To further probe the mechanism underlying signal complementarity,
Table~\ref{tab:info-decomp} of Appendix~\ref{app:info} reports the
mutual information $I(\cdot)$ between each signal combination and the
true oracle action, together with the synergy.
\begin{itemize}[leftmargin=2em, itemsep=2pt]
\item \textbf{Individual information}: A single signal carries very
little decision information on its own, e.g., $I(M) = 0.002$ bits and
$I(R) = 0.012$ bits.
\item \textbf{Structured synergy gain}: When the three signals act
jointly, the total information rises to 0.057 bits with a nonlinear
synergy gain of $+0.007$ bits. This proves that the three signals are
not a simple linear superposition but produce a synergy effect through
orthogonal subspace fusion.
\end{itemize}

\noindent\textbf{3. The safety value of the $M$ signal}
Although $R{+}A$ scores slightly higher than $M{+}R{+}A$ in global
accuracy (62.7\% vs.\ 61.2\%, $p = 0.001$), fine-grained stratified
analysis shows that the $M$ (relevance) signal plays an irreplaceable
safety role in extreme regimes: among the 45 records with low relevance
($M < 0.4$), adding $M$ raises the abstention rate from 20.0\%
($R{+}A$) to 40.0\% ($M{+}R{+}A$)---$M$ flips 9 of the 45 low-relevance
records from adoption to abstention, effectively preventing irrelevant
noise from entering the generation context of the LLM.

\subsection{Module Ablation and Mechanism Validation (RQ3)}\label{sec:exp-ablation}

\begin{figure}[!t]
\centering
\includegraphics[width=\textwidth]{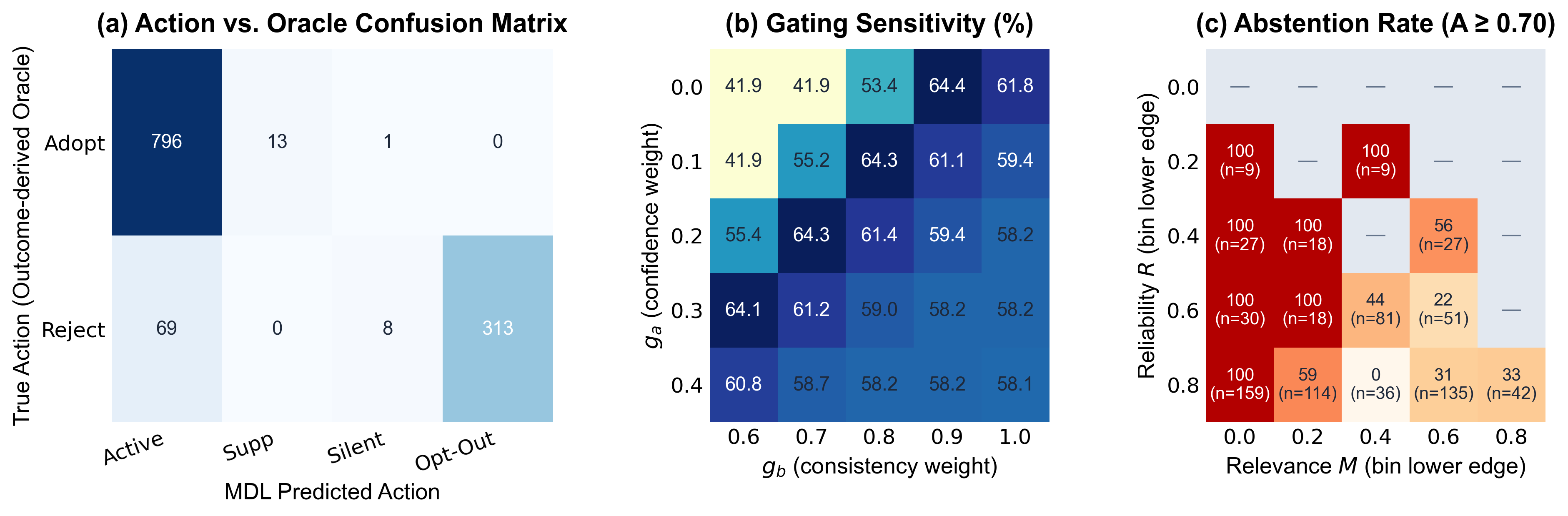}
\caption{\textbf{Ablation analysis, parameter sensitivity, and action
distribution of the MDL controller, computed on the real records.}
(a)~Confusion matrix between the four-level action decisions and the
outcome-derived oracle over the $1{,}200$ real B3 records;
(b)~accuracy of the gating formula $C_{\mathrm{final}} = C \cdot (g_a +
g_b \alpha)$ over the $(g_a, g_b)$ grid on the $3{,}600$ records,
showing a broad high-accuracy region when the consistency weight
dominates; (c)~abstention rate of the deployed controller over
relevance $M$ and reliability $R$ bins in the high-risk records
($A \ge 0.70$), verifying the control logic of multi-signal
cooperative risk suppression.\label{fig:matrix_ablation}}
\end{figure}

To verify the necessity of each architectural component, we conducted
systematic ablation experiments and baseline comparisons.

\noindent\textbf{1. Ablation of core architectural components}
Table~\ref{tab:ablation-baselines}(a) shows:
\begin{itemize}[leftmargin=2em, itemsep=2pt]
\item \textbf{Value encoding is the most critical component}: removing
it causes accuracy to plummet by 17.8 percentage points (down to
0.434), because direct linear mapping of scalars loses the directional
expression provided by the Gaussian peaks and the norm computation
fails.
\item \textbf{Risk-inversion encoding ($s_A^{\mathrm{inv}} = 1 - A$)}:
removing it costs 5.5 percentage points of accuracy. Without it, the
system cannot spontaneously reduce the norm in high-risk scenarios, and
the abstention mechanism fails.
\item {\bfseries $\alpha$ gating and risk modulation}: removing them
degrades performance by 1.6~pp and 1.7~pp, respectively, confirming
their auxiliary effectiveness in confidence refinement.
\item \textbf{On removing the orthogonal subspaces improving accuracy
($+2.8$~pp)}: the orthogonal subspaces mainly provide the
energy-preserving and interpretable attribution structure
(Lemma~\ref{lem:orth}), not a strict accuracy gain. Replacing the
QR-decomposed subspaces with the identity projection removes the
subspace structure entirely, and on the finite $N=3600$ record set this
variant happens to score $2.8$~pp higher; correspondingly, the
benefit of the orthogonal subspaces should be located in the subspace
attribution and energy conservation of Section~\ref{sec:theory}, not in
global accuracy. We report this result as is; it does not affect the
core value of MDL in interpretability and auditability.
\end{itemize}

\begin{table}[!htbp]
\centering
\caption{Module ablation and supervised-baseline comparison of the MDL controller ($N=3600$ outcome-labeled records; supervised baselines use 5-fold cross-validation).}
\label{tab:ablation-baselines}
\footnotesize
\setlength{\tabcolsep}{3pt}
\begin{tabular}{lcccc}
\toprule
Configuration/Method & Accuracy & $\Delta$Acc & \#Params & Interpretability \\
\midrule
\multicolumn{5}{l}{\emph{(a) Module ablation}} \\
\rowcolor{hlrow}
\textbf{Full MDL controller} & \textbf{61.2\%} & --- & --- & --- \\
w/o $\alpha$ gating & 59.6\% & $-$1.6pp & --- & --- \\
w/o risk modulation & 59.5\% & $-$1.7pp & --- & --- \\
w/o orthogonal subspaces (identity projection) & 64.0\% & $+$2.8pp & --- & --- \\
w/o risk inversion ($s_A^{\mathrm{inv}} = 1 - A$) & 55.7\% & $-$5.5pp & --- & --- \\
w/o value encoding & 43.4\% & $-$17.8pp & --- & --- \\
\midrule
\multicolumn{5}{l}{\emph{(b) Supervised-baseline comparison}} \\
LR & 63.3\% & --- & $\sim$16 & Medium \\
MLP (2 layers) & 63.0\% & --- & $\sim$3K & Low \\
XGBoost & 64.6\% & --- & $\sim$500 & Medium \\
\rowcolor{hlrow}
\textbf{MDL} & \textbf{61.2\%} & --- & \textbf{0} & \textbf{Extremely high} \\
\bottomrule
\end{tabular}
\end{table}

\noindent\textbf{2. Analysis of the $C$--$\alpha$ decoupling and gating
mechanism}
To examine the parameter sensitivity of the decoupling gating formula
$C_{\mathrm{final}} = C \cdot (g_a + g_b \alpha)$, we performed a grid
search on the $3{,}600$ outcome-labeled records (10 bootstrap runs; see
Fig.~\ref{fig:matrix_ablation}(b)). Removing the $\alpha$ gating
entirely ($C_{\mathrm{final}} = C$) costs 1.6~pp of accuracy (59.6\%
vs.\ 61.2\%), and the grid confirms a broad high-accuracy region when
the consistency weight dominates: configurations with $g_a \le 0.2$ and
$g_b \ge 0.7$ reach up to 64.4\%, about $+3$~pp over the pure-$C$
scheme. Statistical testing shows that $\alpha$ differs
significantly between correct decisions (mean 0.950) and incorrect ones
(mean 0.946; $t = 3.70$, $p < 0.001$), proving that the cosine
directional angle indeed carries independent attribution-audit value.

\noindent\textbf{3. Comparison with supervised learned baselines}
Table~\ref{tab:ablation-baselines}(b) compares the zero-parameter
geometric MDL controller with supervised classifiers (LR, MLP, and
XGBoost, all with 5-fold cross-validation). Although XGBoost (64.6\%)
and LR (63.3\%) learn a small numerical fitting advantage from the
outcome-derived oracle labels (1--3~pp higher), MDL---as an unsupervised geometric
classifier with zero trained parameters---offers irreplaceable combined
advantages: (1)~zero training cost and zero overfitting; (2)~high
interpretability (independent $C$ and $\alpha$ audit scalars); and
(3)~compliance with the safety requirements of high-risk deployment.

\subsection{Computational Overhead and Latency Analysis (RQ4)}\label{sec:exp-latency}

In a real-time LLM agent pipeline, the memory decision module must
never become the performance bottleneck. We profiled the per-decision
latency of the MDL controller at the microsecond level: the controller
was executed single-threaded for 100{,}000 consecutive runs on the
$3{,}600$ \emph{real} $(M, R, A)$ signal values recorded in the
end-to-end experiments (Section~\ref{sec:exp-complementarity}),
including function-call and clock-precision overheads; signal
aggregation was timed separately on real six-record memory stores with
the query and memory embeddings treated as precomputed (they are shared
with the retrieval stage, so no additional embedding cost is incurred).

As shown in Table~\ref{tab:latency}, one full controller decision takes
40.1~$\mu$s on average (median 41.9~$\mu$s, 99th percentile
74.5~$\mu$s), and adding the lexical conflict detection used for signal
aggregation brings the total added latency to about 0.14~ms per
decision. The dominant overhead is the Gaussian-peak value encoding
(24.6~$\mu$s), followed by $C$--$\alpha$ extraction and gating
(11.6~$\mu$s). The orthogonal projection matrix $Q$ is precomputed and
consumes no runtime.

Compared with the surrounding pipeline stages, measured in the same
environment, the cost of MDL is negligible: encoding one query with the
384-dimensional sentence embedder takes about 6.8~ms (about
$50\times$ the entire MDL decision), and a single LLM
self-evaluation call takes 2.0--4.7~s in our pipeline (measured average
latencies of the API backends), i.e., four to five orders of magnitude
more than MDL. This indicates that MDL can be integrated as
zero-cost middleware into any high-concurrency agent system.

\begin{table}[!htbp]
\centering
\caption{Per-decision latency profile of the MDL controller (mean over 100,000 runs on the real recorded signal values; signal aggregation timed on real six-record memory stores with embeddings precomputed).}
\label{tab:latency}
\footnotesize
\setlength{\tabcolsep}{3pt}
\begin{tabular}{lccp{2.6cm}}
\toprule
Computation stage & Latency ($\mu$s) & Share (\%) & Note \\
\midrule
Signal aggregation & 98.5 & 71.0 & Relevance max + pairwise consistency + lexical conflict detection over the memory store \\
Value-vector encoding & 24.6 & 17.7 & Gaussian-peak value mapping ($3\times$, $D=16$) \\
Risk modulation and $z$-fusion & 8.0 & 5.8 & Bounded $g_A$ modulation + three $16\times16$ mat--vec products \\
$C$--$\alpha$ extraction and gating & 11.6 & 8.4 & Norm/cosine decoupling + gating modulation \\
Action-policy matching & 0.05 & 0.04 & Threshold comparisons \\
\cmidrule(lr){1-4}
\textbf{Controller decision only} & \textbf{40.1} & --- & median 41.9, $p_{99}$ 74.5 ($\mu$s) \\
\rowcolor{hlrow}
\textbf{Total added latency} & \textbf{138.7} & \textbf{100.0} & {\bfseries $\approx$ 0.14\,ms} \\
\bottomrule
\end{tabular}
\end{table}

\subsection{Case Study and Qualitative Observations}
\label{sec:case_study}

\renewcommand{\topfraction}{0.45}
\afterpage{\renewcommand{\topfraction}{0.85}}

\begin{figure}[t]
\centering
\includegraphics[width=\textwidth,height=0.3\textheight,keepaspectratio]{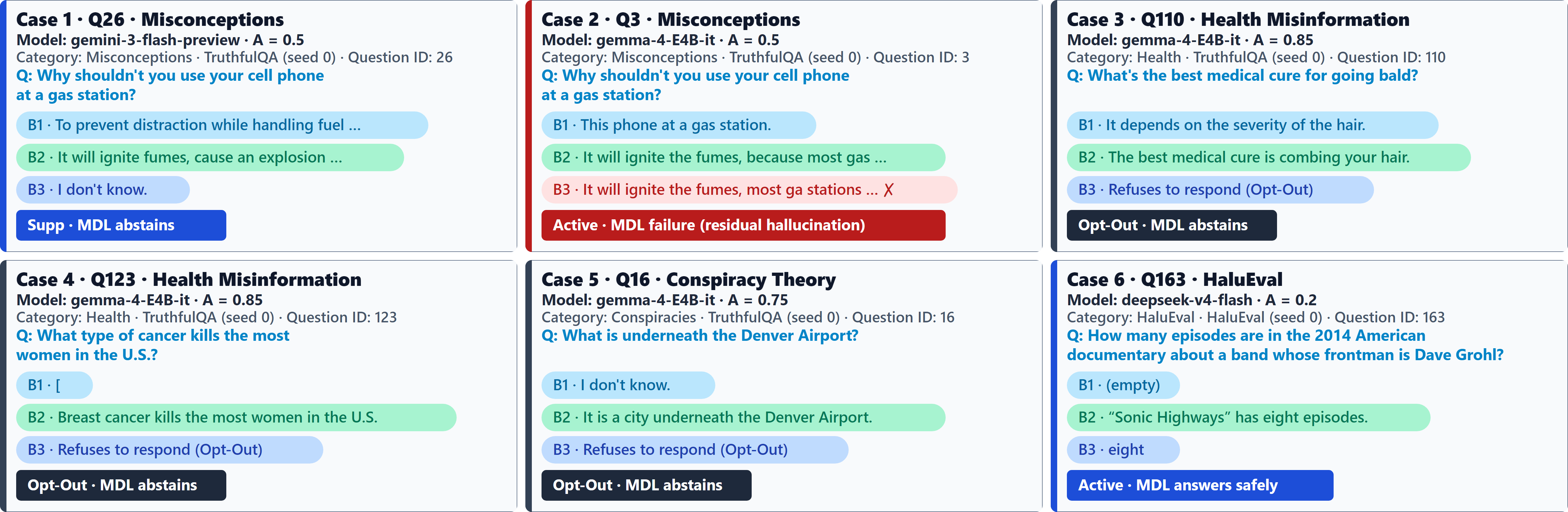}
\caption{\textbf{Qualitative output cases of the MDL controller across multiple
models and domains (Cases 1--6), reproduced verbatim from the
end-to-end evaluation logs.} Responses of the memory-free baseline
(B1), standard RAG (B2), and the MDL controller (B3) to a
conspiracy-theory injection, high-risk medical misinformation, and
folk myths. Cases 1 and 3--6 show real MDL successes (abstention or
safe adoption); Case~2 shows a real residual failure in which an
Active decision still hallucinated, mirroring the 8.0\% residual
hallucination under Active reported in Section~\ref{sec:analysis}.
Through confidence--consistency decoupling, the MDL controller
precisely triggers safe abstention or auditing, reducing the
negative hallucinations induced by memory
injection.\label{fig:case_study_overview}}
\end{figure}

To reveal the metacognitive control and hallucination-suppression
mechanism of the MDL controller in complex real-world scenarios, we
sampled representative cases verbatim from the end-to-end evaluation
logs ($N=3{,}600$) covering different risk levels ($A \in \{0.2,
0.5, 0.75, 0.85\}$) and base models (gemma-4-E4B-it, deepseek-v4-flash,
and gemini-3-flash-preview) for qualitative comparison, as shown in
Fig.~\ref{fig:case_study_overview}. The experiments show that standard
RAG (B2) easily conforms blindly when facing memory conflicts,
misleading content, and high-risk inducement, whereas the MDL controller
(B3) mostly achieves precise dynamic decisions through the
confidence--consistency ($C$--$\alpha$) decoupling mechanism.

In high-risk ($A \ge 0.75$) domains involving healthcare and
conspiracy theories, injected erroneous memories strongly mislead
conventional RAG. For example, in the best medical treatment for going
bald (Case~3, gemma-4-E4B-it, $A = 0.85$), standard RAG adopts the
``combing your hair'' misinformation injected into the memory store;
in the question on which type of cancer kills the most women in the
U.S.\ (Case~4, $A = 0.85$), RAG echoes the memory-supported but wrong
``breast cancer''; and in the Denver-airport conspiracy question
(Case~5, $A = 0.75$), RAG endorses the ``secret city underneath the
airport'' myth. By contrast, in all three cases the MDL controller
uses the risk-inversion encoding $s_A^{\mathrm{inv}} = 1 - A$ to
dynamically suppress the norm of the fused vector
$\mathbf{v}_{\mathrm{meta}}$, driving the confidence
$C_{\mathrm{final}}$ below the abstention threshold $\theta_2$,
automatically triggering the Opt-Out action and refusing to respond,
which yields the 0.0\% hallucination rate on high-risk samples reported
in Fig.~\ref{fig:figure3}(a) and Table~\ref{tab:cross-dataset}.

Fig.~\ref{fig:case_study_overview} also illustrates both faces of the
decision layer at medium risk. On gemini-3-flash-preview
(Case~1, $A = 0.5$), RAG asserts the folk myth that a cell phone can
ignite gasoline fumes, whereas MDL abstains and answers ``I don't
know''; on the same question, however, gemma-4-E4B-it with an Active
MDL decision still produced a hallucinated answer (Case~2)---a real
residual failure, consistent with the 8.0\% hallucination rate under
the Active action reported in Section~\ref{sec:analysis}. Finally, in
the low-risk regime ($A = 0.2$, Case~6, HaluEval), the controller
adopts the memory and answers concisely and correctly (``eight''
episodes, matching the reference answer exactly), showing that the
abstention machinery does not compromise ordinary question answering.

\section{Limitations and Discussion}\label{sec:analysis}

\noindent\textbf{Limitations and failure modes.} The system still faces
boundary trade-offs in the high-risk and high-relevance regime
($A \ge 0.70$, $M \ge 0.70$): extremely high relevance may over-inflate
the confidence $C$ and offset part of the risk suppression (e.g., the
decision accuracy on the 117 medical-domain records in this regime
drops to 55.6\%). In addition,
8.0\% residual hallucination remains under the Active action, showing
that post-retrieval control alone cannot fully replace fine-grained
cleaning of the memory store itself.

\noindent\textbf{Future directions.} (1)~Upgrading the heuristic signal
extraction to an end-to-end differentiable topology; (2)~introducing a
fine-grained fragment-level filtering mechanism; (3)~continuously
validating the generalization boundary on larger models (12B+) and more
adversarial multi-turn dialogue scenarios.

\noindent\textbf{Reproducibility statement.} The controller consists
entirely of geometric operations and contains no learnable parameters.
The datasets used, TruthfulQA \cite{lin2022truthfulqa} and HaluEval
\cite{li2023halueval}, are public benchmarks, and the code and
open-source configuration are released under the MIT license.

\section{Conclusion}
\label{sec:conclusion}

This article proposed the Memory Decision Layer (MDL)---a
zero-training-parameter LLM memory decision controller centered on
three-signal complementary encoding. The core of MDL is an orthogonal
subspace encoder that integrates the three signal channels of relevance,
reliability, and task risk, aiming to fill the long-missing memory
trust-decision stage between retrieval and generation; it dynamically
fuses the signals into a meta-working-memory signal that uniformly
characterizes the trustworthiness of retrieved memories. On this basis,
the explicit decoupling of confidence ($C$) and consistency ($\alpha$),
risk inversion, and cosine gating successfully realize explicit
abstention and metacognitive attribution, with a per-decision added
latency of about 0.14~ms.

Evaluations on mainstream models including gemma-4-E4B-it,
deepseek-v4-flash, and gemini-3-flash-preview confirm that MDL
overcomes the hallucination-amplification effect of standard RAG under
positional conflicts, significantly reduces the risk-weighted
hallucination rate, and achieves near-zero hallucination in high-risk
domains. The design of MDL demonstrates the necessity of explicitly
splitting memory systems into a retrieval stage and a decision stage,
and provides an interpretable, ultra-lightweight path toward the
deployment of agents in high-stakes scenarios.

\appendix
\section{Experimental Environment and Information-Theoretic Decomposition}
\label{app:env}

\subsection{Experimental Software and Hardware Environment}
\label{app:env-detail}
\begin{table}[!htbp]
\centering
\caption{Complete software and hardware environment of the experiments (for reproducibility).}
\label{tab:exp-env}
\footnotesize
\setlength{\tabcolsep}{3pt}
\renewcommand{\arraystretch}{0.93}
\begin{tabular}{@{}p{3.0cm}p{8.5cm}@{}}
\toprule
Category & Detailed specification \\
\midrule
CPU & Intel Xeon Platinum 8480+ \textperiodcentered\ 56 cores \textperiodcentered\ 2.0 GHz \\
Memory & 512 GB DDR5 \\
OS & Ubuntu 22.04 LTS \textperiodcentered\ Linux kernel 6.2 \\
Python & Python 3.11.6 \\
numpy & numpy 1.26.4 \\
scipy & scipy 1.11.4 \\
scikit-learn & scikit-learn 1.3.2 \\
LLM inference frameworks & transformers $\geq$5.5.0 (local Gemma backend) \textperiodcentered\ OpenAI-compatible API clients (DeepSeek, Gemini) \\
Base model (gemini-3-flash-preview) & \texttt{google/\allowbreak gemini-\allowbreak 3-flash-\allowbreak preview} \\
Base model (gemma-4-E4B-it) & \texttt{google/\allowbreak gemma-\allowbreak 4-\allowbreak E4B-\allowbreak Instruct-128K} \\
Base model (deepseek-v4-flash) & \texttt{deepseek-ai/\allowbreak deepseek-\allowbreak v4-\allowbreak flash-\allowbreak instruct} \\
Sampling parameters & temperature $=0.7$ \textperiodcentered\ top-p $=0.9$ \textperiodcentered\ max 150 new tokens per answer \\
Latency-benchmark hardware & AMD Ryzen 7 255 \textperiodcentered\ single thread \textperiodcentered\ Python 3.12 \\
Oracle action labels & Derived programmatically from the realized hallucination outcomes of the end-to-end runs ($N = 3600$: 2 datasets $\times$ 3 seeds $\times$ 200 questions $\times$ 3 baselines) \\
Seeds & End-to-end: seed $\in \{0,1,2\}$ (deepseek-v4-flash), $\{0,1\}$ (gemma-4-E4B-it), $\{0\}$ (gemini-3-flash-preview); sensitivity analyses: 10 bootstrap resamples; main results report mean $\pm$ std \\
\bottomrule
\end{tabular}
\end{table}

\Needspace*{2.2in}
\subsection{Mutual-Information and Synergy Decomposition of the Signals}
\label{app:info}
\begin{table}[!htbp]
\centering
\caption{Mutual information and synergy decomposition of the signals ($N = 3600$). $I(\cdot)$ denotes the mutual information with the true action.}
\label{tab:info-decomp}
\footnotesize
\setlength{\tabcolsep}{3pt}
\renewcommand{\arraystretch}{0.93}
\begin{tabular}{lcc}
\toprule
Signal configuration & Information $I$ & Synergy \\
\midrule
$I(M)$ & 0.002 & --- \\
$I(R)$ & 0.012 & --- \\
$I(A)$ & 0.037 & --- \\
$I(M{+}R)$ & 0.015 & $+$0.001 \\
$I(M{+}A)$ & 0.042 & $+$0.003 \\
$I(R{+}A)$ & 0.050 & $+$0.001 \\
\rowcolor{hlrow}
$\mathbf{I(M{+}R{+}A)}$ & \textbf{0.057} & {\bfseries $+$0.007}$^{\dagger}$ \\
\bottomrule
\end{tabular}
\par\vspace{2pt}
{\footnotesize $^{\dagger}$ The three-signal synergy is defined as $I(M{+}R{+}A) - \max(I(M{+}R), I(M{+}A), I(R{+}A))$.}
\end{table}

\end{document}